\documentclass[runningheads]{llncs}
\usepackage[color=gray!30,textsize=tiny]{todonotes}
\usepackage[T1]{fontenc}

\usepackage[space, compress, sort]{cite}

\usepackage{amssymb} 
\usepackage{amsfonts}
\usepackage{amsthm}
\usepackage{wrapfig}
\usepackage[ruled,vlined]{algorithm2e}
\usepackage{mathtools}
\usepackage{tabularx}
\usepackage{graphicx}
\usepackage{enumerate}
\usepackage{float}
\usepackage{url}
\usepackage{verbatim}
\usepackage{booktabs} 
\usepackage{multirow}
\usepackage[linkcolor=blue,citecolor=blue,urlcolor=blue,colorlinks=true]{hyperref}
\usepackage{bm}
\usepackage{algpseudocode}

\SetKwRepeat{Do}{do}{while}

\begin{document}

\mainmatter

\title{STAR-Filter: Efficient Convex Free-Space Approximation via Starshaped Set Filtering in Noisy Environments
}
\titlerunning{STAR-Filter for Convex Free-Space Approximation} 

\author{Yuwei Wu\textsuperscript{1}, Yichen Zhao\textsuperscript{2}, Dexter Ong\textsuperscript{1},  Vijay Kumar\textsuperscript{1}}
\authorrunning{Y. Wu et al.} 
\institute{\textsuperscript{1}University of Pennsylvania, Philadelphia, PA, 19104 USA\\
\email{\{yuweiwu, odexter, kumar\}@seas.upenn.edu} \\
\textsuperscript{2}Georgia Institute of Technology, Georgia, GA 30332, USA\\
\email{yzhao654@gatech.edu} \\
}

\maketitle

\begin{abstract}
Approximating collision-free space is fundamental to robot planning in complex environments.
Convex geometric representations, such as polytopes and ellipsoids, are widely employed due to their structural properties, which can be easily integrated with convex optimization.
Iterative optimization-based inflation methods can generate large volume polytopes in cluttered environments, but their efficiency degrades as the obstacle set becomes more complex or when sensor data are noisy. 
These methods are also sensitive to initialization and often rely on accurate geometric models.
In this paper, we propose the STAR-Filter, a lightweight framework that employs starshaped set construction
as a fast filter for convex region generation in collision-free space. 
By identifying obstacle points as active supporting constraints, the proposed method significantly reduces redundant computation while preserving feasibility and robustness to sensor noise.  
We provide theoretical and numerical analyses that characterize the structural properties of the starshaped set and proposed pipeline in environments of varying complexity.
Simulation results show that the proposed framework achieves the lowest computation time and reduces conservativeness in polytope generation for real-world noisy and large-scale data.
We demonstrate the effectiveness of the framework for Safe Flight Corridor (SFC) generation and agile quadrotor planning in noisy environments.
\keywords{Starshaped construction, Convex region generation, Collision avoidance}

\end{abstract}

\section{Introduction}

Effective motion planning relies on compact and computationally tractable representations of collision-free space for collision avoidance~\cite{lavalle2006planning, doi:10.1177/027836498400300403}.
For onboard systems that rely on real-time sensory data, the choice between obstacle-centric and free-space-centric parameterizations, as well as the selection of the specific form of free-space representation, critically affects the efficiency, robustness, and scalability of planning and control algorithms. 
Direct discretized representations, such as occupancy grids, signed distance fields, and related volumetric maps, are widely used for localization and mapping~\cite{8202315}.
However, their computational and memory inefficiency often limits their ability to perform real-time motion planning tasks.
Geometric representations, such as convex or starshaped decomposition or extraction~\cite{5152628, 10.1145/1236246.1236265, 10.1145/800076.802459, 5509263, 10164805}, are commonly employed to enable efficient collision checking and to formulate collision-free constraints within motion planning frameworks. 
Depending on the input data (e.g., pre-existing map, raw sensor measurements), a suitable representation of the free space should balance geometric expressiveness, computational tractability, manageability, and the ability to incorporate incremental sensory updates in real time~\cite {10970076}.

While obstacle-centric modeling is prevalent in computational geometry and computer vision, motion planning often benefits more from direct free-space decomposition, which is typically more compact and better suited to trajectory optimization~\cite{7487283, marcucci2024fast, 9102390, 9120168, 7784290, 9561460, 9812306,  10802782, 9982032, werner2024approximating}.
However, constructing such representations remains inherently challenging, as the continuous free space generally lacks an explicit geometric characterization unless the environment is discretized.
Therefore, most methods approximate free space indirectly by excluding the entire obstacle set, although it can be conservative and may exclude larger feasible convex regions.
Polytope-based approximations of the collision-free space are widely adopted in trajectory-planning frameworks, e.g.,  Graphs of Convex Sets (GCS)~\cite{doi:10.1137/22M1523790}, Safe Flight Corridors (SFC)~\cite{7839930}, and  Forward Reachable Sets (FRS)~\cite{kousik2019safe}.
These convex representations can be overly conservative in nonconvex environments, typically requiring multiple convex sets to achieve adequate coverage. To remedy this, starshaped set representations~\cite{beltagy2000convex} provide a natural alternative that aligns well with ray-based sensing modalities such as LiDAR and depth cameras.
Since starshaped sets directly encode visibility from a reference point, they offer a more expressive geometric model of locally observable free space and can reduce the need for an excessive spatial decomposition.
Prior work has utilized starshaped polytopes as intermediate representations for convex polytope generation~\cite{zhong2020generating} or applied them to visibility-aware aerial inspection and exploration~\cite{liu2022star}.

In this paper, we present the STAR-Filter, a lightweight obstacle filtering framework to accelerate convex polytope generation for collision-free space approximation. 
STAR-Filter leverages starshaped set construction to identify a small subset of obstacle points that directly limit the size of the local free space for restricted convex polytope inflation.
This step significantly reduces computational cost while preserving query seed containment and maintaining collision avoidance in the resulting convex region.
We analyze the geometric properties of the induced starshaped set and show that its radial extreme points provide an effective approximation of the local free-space boundary relevant for ellipsoid inflation. 
The framework is evaluated in both simulated environments and real-world datasets, demonstrating robustness to sensor noise and effectiveness for agile motion planning tasks.

\textbf{Contributions.}
(1) We introduce a starshaped, visibility-based filtering of obstacle point clouds and demonstrate the correspondence between extreme points of the starshaped set and the active constraints of the convex polytope.
(2) We leverage this structure to enable fast convex polytope generation, reducing computational cost while preserving feasibility.
(3) We validate the proposed method through extensive simulations and benchmark comparisons on real-world datasets from diverse sensor sources (e.g., LiDAR and stereo event camera).

\section{Related Works}
Conventional geometric map decomposition methods approximate the environment using convex components at varying levels of fidelity~\cite{10.1145/800076.802459,10.1145/1236246.1236265,wei2022approximate,yasser2016,mahroo2021}.
In robotics, particularly for aerial vehicles, sequences of overlapping convex polytopes are widely adopted~\cite{7487283,marcucci2024fast,9982032,9812306,werner2024approximating,9102390}, as they support efficient trajectory optimization and graph-based global planning.
These convex corridors support piecewise trajectory optimization and efficient graph-based global planning.
Other common choices include axis-aligned boxes~\cite{7487283,marcucci2024fast}, spheres and ellipsoids~\cite{7784290, 9561460, 9812306,  10802782}, rectangular pyramid (RAPPIDS)~\cite{9120168}, and convex polytopes~\cite{9982032,werner2024approximating}, typically constructed from point clouds or occupancy maps.

To extract convex regions from nonconvex environments, \cite{Deits14computinglarge} proposed an algorithm of Iterative Regional Inflation by Semidefinite Programming (IRIS) to generate maximum volume polytopes around convex obstacles given seed points. 
The iterative optimization procedure involved two phases: sequential quadratic programming (SQP) to compute separating hyperplanes with respect to obstacles, followed by semidefinite programming (SDP) to determine the inscribed maximum-volume ellipsoid within this polytope.
However, due to the high computational cost of this approach, subsequent work has focused on improving its efficiency.
The work in~\cite{7839930} proposed a non-iterative version of IRIS and used line segments as an initial guess of the ellipsoid axes. 
Fast iterative region inflation (FIRI) in \cite {10970076} further optimizes each phase and extends it to include a given convex polytope. 
Given voxel maps as inputs,~\cite{9102390} proposed parallel convex inflation of axis-aligned cubes along axis directions by incrementally checking voxel points, which can achieve a resolution near real-time.
An alternative approach for constructing convex regions is proposed in~\cite{7998590}, where they introduce a stereographic projection that maps obstacle points or point clouds into a transformed space and converts them into a convex hull.
The resulting free region is often overly conservative for practical applications, and this limitation is addressed in~\cite{zhong2020generating} by using sphere flipping to generate starshaped regions that are subsequently modified into convex polytopes. 
These methods demonstrate efficient polytope generation and show potential for online planning.

For dynamic or local map construction from raw sensor data, free-space representations should support low-latency incremental updates, remain robust under noisy sensing, and scale efficiently in memory.
Authors in~\cite{9981447} introduce a framework for global polytope map construction with dynamic updates under environmental changes.
However, repeated decomposition of existing polytopes often results in fragmentation into many small regions, reducing overall representational and computational efficiency. 
Other work, such as \cite{latha2024} and \cite{feng2024}, applies starshaped sets to build a motion graph for navigation in 2-dimensional (2D) navigation problems, and in~\cite{10594730}, deformable polygons are used as free-space representations for radar-based sensing in the 2D setting.  
We differ by using a sphere flipping map to filter the noisy raw data, and supporting both 2D and 3D workspaces while extending nonconvex representations to higher dimensions. 
Our approach maintains robustness, computational efficiency, and online updates, which remain challenging for most methods.
Another line of research jointly considers perception and planning by explicitly incorporating state estimation uncertainty into the planning process, which places additional demands on the underlying free-space representation.
For example,~\cite{9863844} integrates safe regions with perception-aware costs to mitigate estimation uncertainty, while~\cite{agrawal2025certifiably} uses odometry drift to refine SFCs.
While the proposed framework focuses on collision-free space generation, it remains compatible with drift-compensated approaches and can be integrated into onboard autonomy pipelines.

\section{Prerequisite}

Let $\mathcal{W} \subset\mathbb{R}^{n}$ denote the workspace and $\mathcal{O}\subset \mathcal{W}$ be a closed obstacle set.
The collision-free space is defined as $ \mathcal{F} \coloneqq \mathcal{W} \setminus \mathcal{O}$.
We consider a robot operating in $\mathbb{R}^n$, with $n\in\{2,3\}$, equipped with a ray-casting sensor (e.g., LiDAR or depth camera) rigidly mounted at the robot origin $q\in\mathbb{R}^n$.
A sensor scan yields a point cloud $X=\{x_i\}_{i=1}^N \subseteq \mathcal{O}$.
Each point $x_i$ corresponds to the first occupied measurement along its sensing ray, implying that every point on the line interval $[q,x_i)$ is collision-free.
Consequently, the visibility region induced by the scan,
$V(q,X) \coloneqq \bigcup_{x_i\in X}[q,x_i)
= \{q+\lambda(x_i-q)\mid x_i\in X,\ \lambda\in[0,1)\}$ is contained in $\mathcal{F}$.
We now recall the notion of starshaped sets.
\begin{definition}[Starshaped set]\label{def:star_cvx}
A set $S\subset\mathbb{R}^n$ is starshaped with respect to $q$ if
\[
x\in S \;\Longrightarrow\; (1-s)q+sx\in S,\quad \forall s\in[0,1].
\]
\end{definition}

A starshaped set w.r.t. any point $q\in \mathcal{F}$ is always a valid conservative approximation for $\mathcal{F}$. Accordingly, we aim to find a starshaped region $S\subseteq \mathcal{F}$ to under-approximate $\mathcal{F}$ while preserving geometric properties. 
The maximal starshaped region w.r.t. $q$ is $S_{\max}(q)=\{x\in \mathcal{F}\mid[q,x]\subseteq \mathcal{F}\}$ can be approximated given a point cloud $X$ which satisfies $S_{\max}(q)\supseteq V(q,X)$.

Several inverse mappings (e.g., Kelvin inversion) preserve radial directions toward the query point (e.g., the robot origin), so the inverse-mapped points form a starshaped set.
However, the numerical stability can vary significantly in practice.
We now revisit starshaped construction in~\cite{zhong2020generating,liu2022star} via sphere flipping that has been implemented with point cloud inputs and formalize this approach. 
Consider the map
\begin{equation}~\label{eq:spherical_flip}
   f_R(x)=x-q+2\big(R-\|x-q\|\big)\frac{x-q}{\|x-q\|}, \quad \forall x\in \mathbb{R}^n\setminus \{ q\},
\end{equation}
and if $x=q,f_R(x) = q$ such that $f_R$ is well-defined everywhere. Given a point cloud $X$, we define the map such that $R \ge \sup_{x_i\in X}\|x-q\|$ so the point cloud gets flipped outside of the sphere.
Moreover, $f^{-1}_R(f_R(x))=\operatorname{Id}(x)=x$, where $f_R^{-1}=f_R$.
We show that~\eqref{eq:spherical_flip} yields a starshaped set that preserves visibility.

\begin{figure*}[!t]
      \centering
    \includegraphics[width=1\columnwidth]{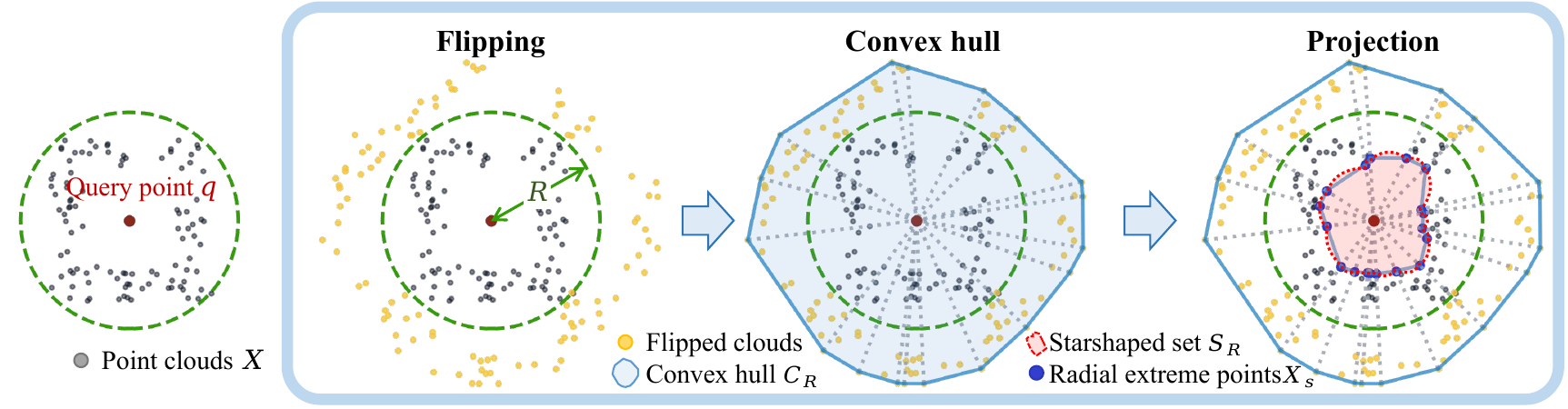}
    \vspace{-0.6cm}
      \caption{Generation of a collision-free starshaped set from point clouds and a query point. Blue points are the extreme points of the starshaped set. }
      \label{fig: p2}
      \vspace{-0.1cm}
\end{figure*}

\begin{proposition}[Starshaped set from sphere flipping]
\label{prop:flip-star}
Let $X=\{x_i\}_{i=1}^N$ be a finite set of points, Let $\bar{B}(q,2R)\coloneqq\{x\in\mathbb{R}^n\mid \|x-q\|\leq 2R\}$ be the closed ball that contains all the flipped points under~\eqref{eq:spherical_flip}, and define
\[
\mathcal{C}_R \coloneqq \operatorname{conv}\bigl(f_R(X\cup\{q\})\bigr),
\qquad
\mathcal S_R \coloneqq \overline{\bar{B}(q,2R)\setminus f_R^{-1}(\mathcal C_R)\,}.
\]
for some $R\geq \sup_{x_i\in X}\|x_i-q\|$, where $\operatorname{conv}$ denotes the convex hull, and $f^{-1}(\mathcal{C}_R)$ denotes the preimage of the set $\mathcal{C}_R$. Then $\mathcal{S}_R$ is starshaped with respect to $q$.
\end{proposition}
\begin{proof}
Without loss of generality, let $q=\mathbf{0}$, then the map becomes 
\begin{equation}~\label{eq:reduced_sphere_flip}
f_R(x)=\frac{2R-\|x\|}{\|x\|}x.
\end{equation}
Expressing $x=rv$ with $r>0$ and $v\in\mathbb{S}^{n-1}$ where $\mathbb{S}^{n-1}\coloneqq\{x\in\mathbb{R}^n\mid\|x\|=1\}$ is the unit sphere,~\eqref{eq:reduced_sphere_flip} can then be written as $f_R(rv)=(2R-r)v$, which shows that the sphere flipping map preserves each ray and extends it to the other side of the sphere. Since $\sup_{x_i\in X}\|x_i\|\leq R$, we have 
$R\leq\|f_R(x_i)\|<2R$, so \(f_R(X)\subset \bar{B}(\mathbf{0},2R)\). Because \(\bar{B}(\mathbf{0},2R)\) is convex and contains the origin, it follows that
$\mathcal{C}_R\subset \bar{B}(\mathbf{0},2R)$. 

Fix \(v\in\mathbb{S}^{n-1}\), and let $\ell_v=\{tv \mid t\ge 0\}$. Since $\mathcal{C}_R$ is convex and $\mathbf{0}\in\mathcal C_R$, there exists $\rho(v)\in[0,2R]$ such that
\[
\mathcal{C}_R\cap \ell_v=\{tv\mid0\le t\le \rho(v)\},
\]
where $\rho\coloneqq\max\{t\in[0,2R]\mid tv\in\mathcal{C}_R\}$. Therefore,
\[
f_R^{-1}(\mathcal C_R)\cap \ell_v
=
\{rv \mid 2R-\rho(v)\le r\le 2R\}.
\]
Taking the complement in $B(\mathbf{0},2R)$ and then taking the closure gives
\[
\mathcal S_R\cap \ell_v
=
\{rv \mid 0\le r\le 2R-\rho(v)\}.
\]
Thus $\mathcal{S}_R$ is an initial radial segment on every ray from the origin. Hence, if \(x\in\mathcal S_R\), then $[\mathbf{0},x]\subset\mathcal{S}_R$. Therefore $\mathcal{S}_R$ is starshaped with respect to $\mathbf{0}$, and thus with respect to $q$ in the sense of Definition~\ref{def:star_cvx}.
\end{proof}

Proposition~\ref{prop:flip-star} shows that $\mathcal{S}_R$ is a starshaped set, and its inscribed polytope is also starshaped with extreme points $X_s$ as in Fig.~\ref{fig: p2}.
We leverage the expressiveness of the starshaped set in Section~\ref{sec:approach} to find an inscribed convex polytope.

\section{Problem Formulation}

We consider the general case in which obstacles are represented by a finite point set.
Obstacles may be represented either as point sets or as polytopes; for polytopic obstacles, it suffices to consider only their vertices.
Given a finite point set $X=\{x_i\}_{i=1}^N \subset \mathbb{R}^n$ to be excluded and a required inclusion query point $q\in\mathbb{R}^n$, we aim to find a convex polytope $\mathcal{P}\subset \mathcal{F}$ with maximum volume that contains $q$ while excluding all obstacle points from its interior:
\begin{subequations}\label{eq:main-problem}
\begin{align}
\max_{\mathcal{P}} \quad & \operatorname{vol}(\mathcal{P}) \label{obj:vol}\\
\text{s.t.}\quad 
& x_i \notin \operatorname{int}(\mathcal{P}), \quad  \forall i \in \{1,\dots,N\}  \quad q \in \mathcal{P}.
\end{align}
\end{subequations}
We transform the polytope towards the query point $q$ by defining the $\mathcal{ \tilde{P}}=\mathcal{P}-q$, so that $\mathbf{0}\in\mathcal{ \tilde{P}}$. 
We then obtain the polytope in the normalized H-representation
\begin{equation}\label{eq:Qdef}
\tilde{\mathcal{P}}=\left\{ y \in \mathbb{R}^n \;\middle|\;
a_j^\top y \le 1,\; \forall j \in \{1,\dots,m\}
\right\}.
\end{equation}
where $\{a_j\}_{j=1}^m \subset \mathbb{R}^n$. 
We assume that $m$ is chosen so that $\mathcal{ \tilde{P}}$ is full-dimensional and bounded. This representation ensures $q\in\mathcal{P}$ by construction.
Let $y_i = x_i - q$, the condition that $x_i$ does not lie strictly inside $\mathcal{P}$ is equivalent to the existence of at least one facet that separates $y_i$, i.e.,
\begin{equation}\label{eq:exclude}
y_i \notin \operatorname{int}(\mathcal{ \tilde{P}})
\quad\Longleftrightarrow\quad
\exists j\in\{1,\dots,m\} \text{ such that } a_j^\top y_i \ge 1.
\end{equation}
We introduce $z_{ij}\in\{0,1\}$ indicating whether separating plane $j$ certifies the exclusion of point $x_i$, and a sufficiently large constant $M>0$, the condition \eqref{eq:exclude} can be rewritten as
\begin{equation}\label{eq:exclude-mip}
\begin{aligned}
\exists z_i \in \{0,1\}^m \ \text{s.t.}\quad
& a_j^\top y_i \ge 1 - M(1-z_{ij}), && \forall j \in \{1,\dots,m\}, \\
& \sum_{j=1}^m z_{ij} \ge 1.
\end{aligned}
\end{equation}
For the objective \eqref{obj:vol}, exact volume computation over a convex polytope is computationally intractable in general.
Hence, we adopt a standard objective of maximizing the volume of an inscribed ellipsoid (inner Löwner–John ellipsoid).
We introduce an ellipsoid $\mathcal{E}
=\left\{ y\in\mathbb{R}^n \,\middle|\,
y = \mu + D \tilde{y},\ \ \| \tilde{y} \| \le 1
\right\} $, centered at $\mu\in\mathbb{R}^n$, where $D \in \mathbb{R}^{n\times n}$ is invertible.
The volume of $\mathcal{E}$ is proportional to maximizing $\log \det D$,
which provides a tight geometric bound on the polytope volume and has a tractable formulation.
Although the polytope $\mathcal{ \tilde{P}}$ is origin-centered, we do not fix the ellipsoid center $\mu$, which is optimized jointly with $D$.
For each halfspace $a_j^\top y \le 1$, the containment condition $\mathcal{E} \subseteq \tilde{\mathcal{P}}$ requires
\begin{equation}\label{eq:containment}
\sup_{\|\tilde{y}\|\le 1} a_j^\top(\mu + D \tilde{y})
=
a_j^\top \mu + \|D^\top a_j\|
\le 1,
\qquad \forall j \in \{1,\dots,m\}.
\end{equation}
We formulate the coupled problem as the following mixed-integer program
\begin{equation}\label{eq:final}
\begin{aligned}
\max_{a,\,\mu,\,D,\,z} \quad & \log\det D \\
\text{s.t.}\quad
&  \eqref{eq:exclude-mip}, \eqref{eq:containment}
\end{aligned}
\end{equation}
The resulting polytope $\mathcal{P}=q+\mathcal{ \tilde{P}}$ contains the point $q$ by construction and excludes all points in the obstacle point set~$X$ from its interior. 
This problem is highly nonconvex due to the joint optimization over the polytope parameters $\{a_j\}$ and the binary variables $\{z_{ij}\}$.
Deits and Tedrake~\cite{Deits14computinglarge} decouple a similar problem via a two-stage iterative framework that alternates between greedily constructing separating hyperplanes and computing a maximum-volume inscribed ellipsoid (MVIE) for a fixed polytope.

There are several possible improvements at each stage of this algorithm.
MVIE is a convex optimization problem and formulated as a semidefinite program (SDP) with conic quadratic constraints  in~\cite{Deits14computinglarge}.
Further improvements reformulate the problem as a second-order cone program (SOCP) for computational efficiency~\cite{10970076}.
For the construction of separating hyperplanes, prior approaches employ a greedy strategy that repeatedly identifies the obstacle point closest to the current ellipsoid, computes the ellipsoid-induced normals, and a supporting hyperplane tangent to a uniformly inflated version of the ellipsoid at that point.
This hyperplane is then added to the polytope to exclude the point outside the free region while preserving the feasibility of the current ellipsoid.
However, this greedy procedure can be conservative and expensive, as it requires evaluating distances to all remaining obstacle points at each iteration. 
The computational cost becomes significant for large-scale point clouds. 
We propose a more efficient and lightweight approach for generating separating hyperplanes that reduces per-iteration complexity while maintaining practical effectiveness.

\section{Approach}\label{sec:approach}

In the iterative optimization framework, the polytope update step leverages the current MVIE as a heuristic to select separating planes that preserve collision-free and enable further ellipsoid expansion.
In practice, however, most obstacle points are redundant and do not become active constraints at the solution, and therefore do not influence either the ellipsoid or the polytope.

\begin{figure*}[!th]
      \centering
    \includegraphics[width=1\columnwidth]{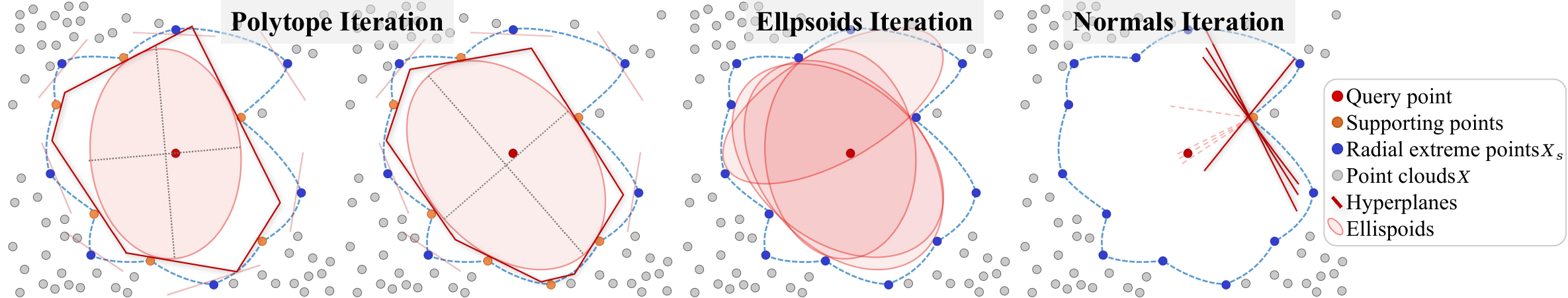}
    \vspace{-0.6cm}
      \caption{
      Iterative refinement of the polytope, ellipsoid, and associated normals. A subset of obstacle points serves as supporting points that determine the active hyperplanes of the polytope, and this subset is contained in the starshaped set.
      }
      \label{fig: p3}
      \vspace{-0.1cm}
\end{figure*}

As shown in Fig.~\ref{fig: p3}, the majority of obstacle points remain inactive throughout the polytope, ellipsoid, and normal update steps. 
The objective of the iteration is the identification of effective separating normals for a small set of supporting points, rather than the construction of an accurate maximum-volume ellipsoid within the collision-free space.
The ellipsoid is not required to be strictly collision-free nor to contain the query seed at every iteration.
This iterative procedure can be interpreted as an optimization over the normals associated with this subset of points using the ellipsoid as a heuristic.
In addition, the intermediate polytopes generated during the iteration only require a sufficiently tight local approximation to support ellipsoid growth, while an exact collision-free coverage at every iteration is not required.

This observation motivates reducing the full obstacle point set to a smaller subset that is the most relevant to optimizing the supporting points' normals. 
The starshaped set construction naturally identifies obstacle points lying on the visible inner boundary of the free space towards the ellipsoid center.
These points are precisely those that can become active supporting constraints, making them a suitable choice for filtering the obstacle point set.

The proposed algorithm is summarized in Alg.~\ref{alg: 1}.
We consider scenarios in which the robot seed is either a point or a line segment, which allows the initialization of a collision-free ellipsoid.
When the seed is a point, we approximate it using a sufficiently short line segment aligned with the major axis of the ellipsoid.

\begin{algorithm}[!ht]
    \LinesNumbered
    \caption{Iterative Inflation with Starshaped Filtering}
    \label{alg: 1}
    \KwIn{query point $q \in \mathcal{F}$, obstacle point set $X \subset \mathbb{R}^n$, ellipsoid parameters $D_0$}
    \KwOut{convex polytope $\mathcal{P}$}
    $\mathcal{E}_0 \leftarrow \textbf{InitEllipsoid}(D_0, q)$  \\
    $X_{s}  \leftarrow $\textbf{StarshapedFiltering}($X, q$)  \\
    $k \leftarrow 1$ \\
    \Repeat{end condition $\lor$ $k \ge k_{\max}$}{
        $\mathcal{P}_k \leftarrow \textbf{UpdateNormals}(\mathcal{E}_{k-1}, X_{s})$\\
        $\mathcal{E}_{k}(D_k, \mu_k) \leftarrow \textbf{MVIE}(\mathcal{P}_k)$ \\
        $k \leftarrow k + 1$
    }
    \If{certify}
    { $\mathcal{P}_k \leftarrow\textbf{GreedySeparatingPlanes}(\mathcal{E}_{k-1}, X)$}
    \Return $\mathcal{P}_k$
\end{algorithm}

Given a query point, we construct a starshaped set and extract its radial extreme points (Line 2).
These points approximate the active supporting point set during polytope refinement.
During polytope generation (Line 5), we do not employ a greedy hyperplane selection strategy.
Instead, when the number of points $X_s$ is moderate, we update the normals of the separating hyperplanes for all supporting points and directly assemble the polytope in H-representation.
This procedure may introduce redundant constraints, however, redundancy removal is only performed when the number of constraints becomes large.
In such cases, we convert the polytope to a vertex V-representation to identify and eliminate redundant hyperplanes.
The MVIE update and end condition follow \cite{Deits14computinglarge}, with a maximum iteration limit $k_{\max}$ for the iterative loop (Line 8).
If a strictly collision-free polytope is required (Line 9-10), we perform a final iteration of greedy separating hyperplane generation to obtain collision-free guarantees.
This framework is particularly effective for environments with large-scale obstacle point clouds, as it avoids repeated global collision checks and expensive greedy updates.
The remainder of this section presents the vertex generation via the starshaped construction and the resulting interior polytope construction.

\subsection{Supporting Point Approximation via Starshaped Construction}

The filtered set is mainly influenced by the choice of the flipping radius $R$, which controls the extent of radial reweighting in the starshaped construction, and the inclusion of boundary points, which affect the completeness of the starshaped approximation near map limits.
In this section, we analyze the geometric and robustness properties of this construction and its variants, including the flipping radius $R$, the treatment of the workspace boundary, the conditions under noisy point clouds, and the ellipsoid-normalized flipping map.

\textbf{Flipping radius}.
For any two points $x_i,x_j\in X$, let $l_i=\|x_i - q \|$ and $l_j=\|x_j - q\|$. 
The radial ratio induced by the flipping satisfies
\begin{equation}
\beta =\frac{R-l_i}{R-l_j}
\leq 1+ \frac{\sup_{x\in X} \|x - q\| - \inf_{x\in X} \|x - q\|}{R}.
\end{equation}
As $R$ increases, relative radial variation among flipped points diminishes and the transformed points concentrate on the sphere with radius $R$ denoted $\mathbb{S}^{n-1}_R$. 
As the limit $R\to\infty$, the flipping map becomes nearly radial-invariant, and the induced starshaped construction approaches the identity mapping. 
The flipping radius~$R$ controls a trade-off between aggressive filtering and geometric approximation.

\begin{proposition}[Monotonicity in $R$]\label{prop:monotone-R}
Let $R_1, R_2 \in \mathbb{R}_{>0}$ satisfy
$R_2\ge R_1 \ge \sup_{x_i\in X}\|x_i\|$, Then $\mathcal{S}_{R_1} \subseteq \mathcal{S}_{R_2}$, where $\mathcal{S}_R$ is defined in Proposition~\ref{prop:flip-star}.
\end{proposition}
\begin{proof}
For any $x_i\in X\setminus \{q\}$, the sphere flipping maps satisfy
\[
f_{R_2}(x_i)=f_{R_1}(x_i)+\frac{2(R_2-R_1)}{\|x_i-q\|}(x_i-q),
\]
thus $f_{R_1}(X)\subseteq \operatorname{cone}(f_{R_2}(X))$, where $\operatorname{cone}(f_{R_2}(X))\coloneqq\{\theta (x_i-q) \mid x_i\in X,\theta\geq 0\}$. Since $R_1\leq R_2$,
taking convex hulls yields $\mathcal{C}_{R_1}\subseteq \mathcal{C}_{R_2}$, applying inverse maps give $f_{R_1}^{-1}(\mathcal{C}_{R_1})\supseteq f_{R_2}^{-1}(\mathcal{C}_{R_2})$, and therefore $\mathcal{S}_{R_1}\subseteq \mathcal{S}_{R_2}$.
\end{proof}

In practice, the flipping radius $R$ is chosen adaptively based on the distribution of obstacle points. 
When multiple starshaped constructions are applied in a localized region, parameter~$R$ can be increased progressively to balance exploration and refinement.
Fig.~\ref{fig: method} (a) shows a numerical example.

\textbf{Regional boundary}.
In starshaped constructions, the set $\mathcal{S}_R$ may under-approximate the collision-free region if $q\notin\operatorname{conv}\big(f_R(X)\big)$, and the projected starshaped set can be a small projected set that does not contain the query point.
In~\cite{liu2022star}, additional points are added around the map boundary to alleviate this issue. 
We define the boundary set as ${\partial \mathcal{W}}\coloneqq\{y\in \mathcal{W}\mid [q,y]\subset \mathcal{F}\}$, now consider the augmented region
\begin{equation}
\mathcal{C}^{+}_R\coloneqq\operatorname{conv}\big(f_R(X\cup\{q\}\cup {\partial \mathcal{W}})\big),
\qquad
\mathcal{S}^{+}_R\coloneqq \overline{\bar{B}(q,2R)\setminus f^{-1}_R \big(\mathcal{C}_R^+ \big)}.
\end{equation}
The augmented set $\mathcal{S}_R^{+}$ remains starshaped with respect to the query point and typically provides a larger and more robust inner approximation of the free space, as can be seen in Fig.~\ref{fig: method}(b).

In the idealized continuous setting, a perfect boundary is invariant under the flipping map and the convex-hull operation.
As checking the containment of the actual boundary is usually more expensive than simply adding discrete points to the obstacle point set, the selection of the augmented points reflects the resolution of the discretization.
For a discrete approximation $\widehat{\partial \mathcal{W}} \subset \partial \mathcal{W}$, invariance holds approximately. 
The quality of the approximation increases as the sampling density increases.
Assuming the boundary is known with a continuous realization, then the procedure described in Fig.~\ref{fig: p2} preserves the boundary set when it is augmented to the original point cloud. 
A discretization procedure can be done on the boundary set to generate points that augment the point cloud $X$. 
A finer discretization produces a better approximation of the free space $\mathcal{F}$. However, it also increases computational cost and remains conservative, still potentially excluding portions of free space.

Therefore, we do not explicitly maintain the full starshaped set, but keep its radial extreme points as interior approximated obstacle points for polytope generation. 
This relaxation avoids sensitivity to boundary discretization while preserving the geometric information for ellipsoid and normals updates.

\begin{figure*}[!t]
      \centering
    \includegraphics[width=1\columnwidth]{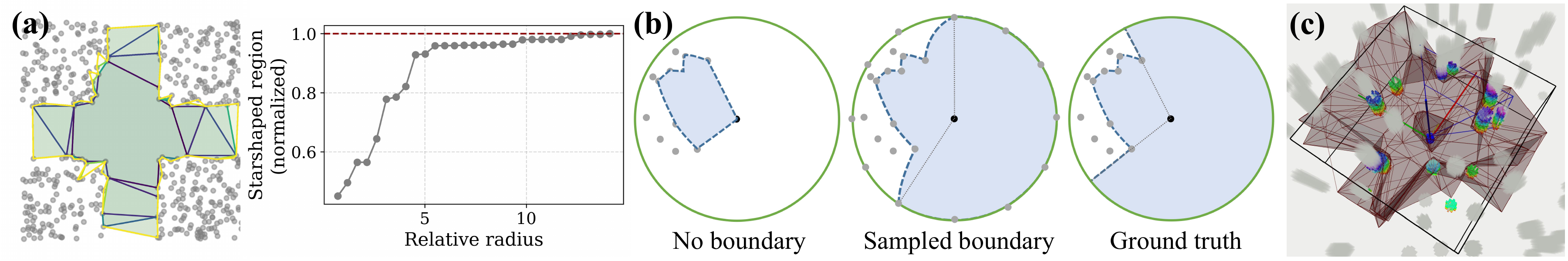}
    \vspace{-0.5cm}
     \caption{Starshaped construction. (a) Monotonic expansion of a 2D starshaped region as $R$ increases; darker to lighter (yellow) indicates increasing $R$. (b) Increased boundary sampling reduces region loss (region in blue). (c) 3D starshaped meshes.}
      \label{fig: method}
\end{figure*}

\textbf{Starshaped filtering under noise}.
Fix the query point $q\in\mathbb{R}^n$. For any point $x_i\in X$ with $x_i\neq q$, define
\[
r_i \coloneqq \|x_i-q\|\in(0,R],\qquad 
\theta_i \coloneqq \frac{x_i-q}{\|x_i-q\|}\in\mathbb{S}^{n-1},
\]
such that $x_i = q + r_i \theta_i$.
Now assume the measurement noise on $x_i$ is $\hat x_i = x_i + w_i, w_i \overset{\text{i.i.d.}}{\sim}\mathcal{N}_n(\mathbf{0},\Sigma)$.
Define $\hat x_i$ along direction $\theta_i$ by~$\hat r_i \coloneqq \theta_i^\top(\hat x_i-q)$. Then $\hat r_i = r_i + \theta_i^\top w_i$, where
\[
\theta_i^\top w_i \sim \mathcal{N}(0,\sigma_{\theta_i}^2),\qquad 
\sigma_{\theta_i}^2 \coloneqq \theta_i^\top \Sigma \theta_i.
\]
Now fix a direction $\theta\in\mathbb{S}^{n-1}$. Let $\bar r(\theta)$ be the true first-hit range and assume all other measurements along this direction satisfy $r_k(\theta)\ge \bar r(\theta)+\delta(\theta)$ with some gap $\delta(\theta)>0$. 
Let $\{\hat r_j\}_{j=1}^{N_1}$ be the measured ranges of first-hit returns and
$\{\hat r_k\}_{k=1}^{N_2}$ be the measured ranges of non-first-hit returns along this $\theta$, i.e.,
\[
\hat r_j = \bar r(\theta)+\varepsilon_j,\qquad
\hat r_k = r_k(\theta)+\eta_k,
\]
where $\varepsilon_j,\eta_k \overset{\text{i.i.d.}}{\sim}\mathcal{N}(0,\sigma_\theta^2)$ with $\sigma_\theta^2\coloneqq \theta^\top\Sigma\theta$.
Let $\hat x(\theta)$ be the selected measured point with the minimum measured range. Proposition~\ref{prop:prob} provides an upper bound on the probability that, due to Gaussian measurement noise, a non-first-hit measurement is incorrectly selected.

\begin{proposition}[Non-first-hit probability bound]\label{prop:prob}
Let $\hat{\mathcal{X}}_1 \coloneqq \{\hat x_j\}_{j=1}^{N_1}$ be the set of measured first-hit points along direction $\theta$. Then
\[
\mathbb{P}\!\left(\hat x(\theta)\notin\hat{\mathcal{X}}_1\right)
\le
N_1N_2\,\Phi\!\left(-\frac{\delta(\theta)}{\sqrt{2}\sigma_\theta}\right)
\le
\frac{N_1N_2\sigma_\theta}{\sqrt{\pi}\,\delta(\theta)}
\exp\!\left(-\frac{\delta^2(\theta)}{4\sigma_\theta^2}\right),
\]
where $\Phi(\cdot)$ is the cumulative distribution function (CDF) of a standard normal distribution.
\end{proposition}

\begin{proof}
The condition $\hat x(\theta)\notin\hat{\mathcal{X}}_1$ requires $\hat r_k<\hat r_j$ for some pair $(j,k)$.
Since $r_k(\theta)\ge \bar r(\theta)+\delta(\theta)$, $\hat r_k<\hat r_j \Rightarrow \eta_k-\varepsilon_j<-\delta(\theta)$.
Because $\eta_k-\varepsilon_j\sim \mathcal{N}(0,2\sigma_\theta^2)$,
\[
\mathbb{P}(\hat r_k<\hat r_j)
\le \Phi\!\left(-\frac{\delta(\theta)}{\sqrt{2}\sigma_\theta}\right).
\]
A union bound over $N_1N_2$ pairs yields the first inequality. The second inequality follows from~$\Phi(-t)\le \frac{1}{\sqrt{2\pi}t}\exp(-t^2/2)$ for $t>0$.
\end{proof}

\textbf{Biased line-segment containment}.
Let $\mathcal E_k = \{\mu_k + D_k \tilde{y} \mid \|\tilde{y}\| \le 1 \}$
be an ellipsoid (e.g., the current MVIE of the polytope) that contains the line segment,
and define the ellipsoid-normalized coordinate $u(x):= D_k^{-1}(x-\mu_k)$,
where $\mathcal E_k$ is mapped to the unit ball.
The distance of a point $x$ in this normalized coordinate is $\|u(x)\|$.
We define an \emph{ellipsoid-normalized flipping map} parameterized by $R>1$:
\begin{equation}\label{eq:ellip-flip}
g_R(x)
\coloneqq
\mu_k
+
D_k\left( \big(2R-\|u(x)\|\big)\frac{u(x)}{\|u(x)\|} \right),
\qquad x\neq \mu_k.
\end{equation}
This map reflects points in ellipsoid-normalized space across the level set $L_R(u)\coloneqq\{x\in\mathbb{R}^n\mid u(x)=R\}$ and maps them to the original space. 

An example of a 3D starshaped polytope is shown in Fig.~\ref{fig: method}(c).
When the query point or environment changes incrementally, the starshaped representation can be updated locally. 
We can test whether the new query point lies within an approximation of the star kernel, computed as the intersection of halfspaces induced by the original supporting planes. 
If satisfied, the starshaped property is preserved, and only local refinements are performed to incorporate newly observed obstacles; otherwise, the starshaped set is reconstructed from the new query position.

\subsection{Interior Polytope Generation}

With a starshaped approximation of the local collision-free region, we derive a convex polytope that provides an inner approximation of the region.
In~\cite{zhong2020generating}, the starshaped set is approximated by a nonconvex polytope and then convexified by adding separating hyperplanes that exclude interior obstacle points.
This is overly conservative and loses a significant amount of free-space volume.

Given the starshaped extreme point set $X_s$ as shown in Fig.~\ref{fig: p2}, we construct the interior polytope
by intersecting supporting halfspaces induced by each $x_i \in X_s$.
We define the outward normal in the ellipsoid normalized coordinate as
$a_i= u(x_i)/\|u(x_i)\|$ and the associated halfspace
$\mathcal{H}_i= \{ y \in \mathbb{R}^n \mid a_i^\top (y - x_i) \le 0 \}$,
which excludes $x_i$ while preserving containment of the current ellipsoid.
The resulting polytope $\mathcal{P}_k= \bigcap_{x_i \in X_s} \mathcal{H}_i$ provides a valid
inner approximation of the collision-free space.
Unlike greedy IRIS-style methods that add a single separating constraint per
iteration, all candidate supporting halfspaces induced by $X_s$ are incorporated in the subsequent MVIE update.

\vspace{-0.2cm}
\section{Results}

\subsection{Implementation Details}

\begin{wrapfigure}{r}{0.41\columnwidth}
    \vspace{1em}
    \centering
    \vspace{-0.5cm}
\includegraphics[width=\linewidth]{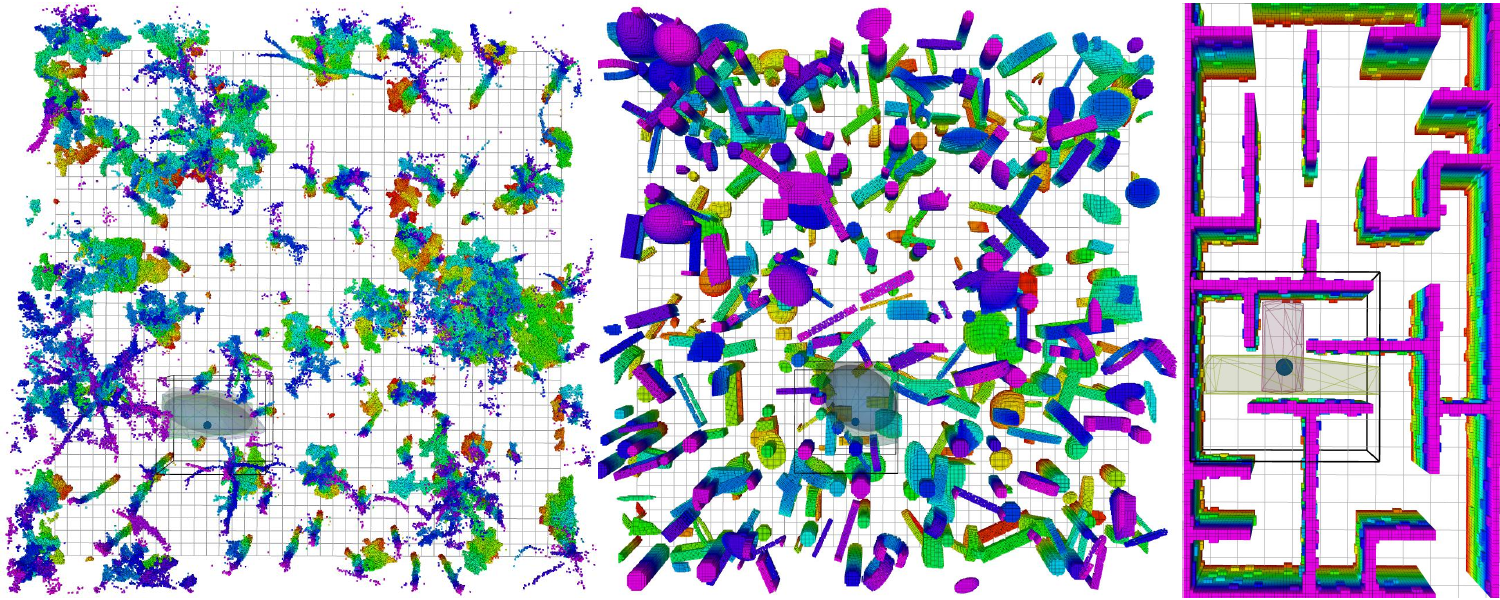}
\vspace{-0.9cm}
\caption{Evaluation maps (left to right): real, obstacle, maze.}
\label{fig:pc}
\vspace{-0.3cm}
\end{wrapfigure}

We evaluate the proposed framework using the benchmark pipeline in~\cite{10610207} using diverse environments. 
As shown in Fig.~\ref{fig:pc}, we consider maze environments (denoted as maze), randomly generated obstacle fields (denoted as obstacle), and real-world point clouds (denoted as real) derived from the M3ED dataset~\cite{Chaney_2023_CVPR}.
For the simulated environment, we consider two map representations. 
Raw uses raw point clouds augmented with zero-mean Gaussian noise to model sensing uncertainty. 
Surface Grid uses a voxelized surface representation of inflated obstacles and does not include additional noise. 
For the real-world dataset, we only use raw point clouds.
We set the obstacle and real map as size of $40 \ \textrm{m}  \times 40 \ \textrm{m}  \times 10 \ \textrm{m} $, and the maze as $20 \ \textrm{m}  \times 10 \ \textrm{m}  \times 5 \ \textrm{m} $.
In each test, a local map is cropped with a range of 3m.
For starshaped construction, we choose $R = \gamma \cdot \max_{x_i\in X} \|u(x_i)\|,  \gamma \in [1,\,5]$ to ensure sufficient radial separation and augment $X$ with the vertices of the local map's bounding box.

For benchmarking, we compare our method with FIRI~\cite{10970076}, RILS~\cite{7839930}, and Galaxy~\cite{zhong2020generating}.
For FIRI, we adopt the lite implementation (denoted as FIRI-lite)
to address the seed containment issue in IRIS and improve the MVIE computation.
FIRI is significantly faster than IRIS and achieves computation time comparable to RILS and Galaxy.
In this lite version, the seed is represented as a point or a line segment, and the restrictive half-space computation for polytope containment is neglected.
We employ Quickhull~\cite{10.1145/235815.235821} for convex hull construction in both Galaxy and our method.
Galaxy directly generates convex hulls from starshaped polytopes and meshes.
To reduce collision violations, Galaxy typically requires a higher resolution in the convex hull construction.
However, even with increased resolution, strict collision-free guarantees cannot be ensured.
The dominant runtime in Galaxy arises from repeated convex hull construction, and the performance of Quickhull degrades as the effective sampling radius increases, since a larger radius introduces more obstacle points.
Galaxy needs an inherent trade-off between sampling resolution, computational cost, and practical collision avoidance.
In contrast, our method does not require the starshaped region to be collision-free during its construction.
This allows the proposed method to apply a low-precision convex hull with significantly reduced computation cost while maintaining effectiveness in iterative optimization.

\subsection{Performance Analysis}

We evaluate computational performance and different initializations across diverse environments, with 200 runs per setting.

\textbf{Computation time}.
We evaluate runtime performance using randomly generated obstacle maps with varying parameterizations. 
Three obstacle density levels are considered using raw point clouds.
The computation time of each component is shown in Table~\ref{tab:time}.
\begin{table}[!ht]
\scriptsize
\renewcommand\arraystretch{1.1}
\setlength{\tabcolsep}{2pt}
\centering
\caption{Computation time (ms) for three density levels of obstacle map.}
\label{tab:time}
\begin{tabular}{lccc}
\toprule
\textbf{Component} 
& \textbf{Dense} 
& \textbf{Medium} 
& \textbf{Sparse} \\
\midrule
\#Local points
& 36771 $\pm$ 11192
& 13156 $\pm$ 9437
& 4564 $\pm$ 7216 \\
\midrule
Starshaped construction       
& 0.979 $\pm$ 0.362
& 0.338 $\pm$ 0.242
& 0.210 $\pm$ 0.200 \\
Iteration in total        
& 1.054 $\pm$ 0.369
& 0.698 $\pm$ 0.257
& 0.722 $\pm$ 0.478 \\
\quad Normal updates (per-iter)
& 0.0004 $\pm$ 0.0002
& 0.0003 $\pm$ 0.0001
& 0.0006 $\pm$ 0.0003 \\
\quad MVIE (per-iter)       
& 0.168 $\pm$ 0.029
& 0.138 $\pm$ 0.023
& 0.225 $\pm$ 0.135 \\
Greedy separating planes        
& 1.366 $\pm$ 0.583
& 0.373 $\pm$ 0.341
& 0.165 $\pm$ 0.214 \\
All                            
& 3.398 $\pm$ 1.062
& 1.410 $\pm$ 0.725
& 1.097 $\pm$ 0.744\\
\bottomrule
\end{tabular}
\end{table}

The cost of greedy computing separating planes over the full obstacle point set increases significantly with the number of obstacle points. 
In dense environments, the greedy separating planes step becomes the dominant part of the total runtime.
The proposed starshaped construction mitigates this cost by reducing the effective point set used in iterations. 
After the starshaped region is constructed, the point subsets involved in normal updates (typically fewer than 40 points) are smaller and the computation time is negligible across all density levels.
Although the convex hull computation for starshaped construction introduces some cost, it is comparable to a single greedy separating plane computation and is allocated over the iterative process.

\textbf{Ellipsoid initialization}. 
If the seed is a query point, previous works typically initialize the ellipsoid as a small sphere centered at the seed, with a radius comparable to the map resolution.
When the seed is a line segment, the initialization is less obvious and has not been explicitly discussed in prior work. 
In this case, different choices can affect the behavior of the optimization. 
Initializing the ellipsoid with its major axis aligned to the line segment may overly constrain the feasible region at early stages, leading to conservative behavior, such as in RILS. 
To analyze this effect, we consider three initialization modes: a small spherical initialization to point seeds (denoted as small), a large spherical initialization whose radius is half the segment length (denoted as large), and a thin segment-aligned initialization that produces an elongated ellipsoid along the line direction (denoted as thin). 
Here, we use RIFI-lite as a reference and test in maze and obstacle maps with raw point clouds. 
\begin{figure*}[!th]
      \centering
       \vspace{-0.1cm}
    \includegraphics[width=1\columnwidth]{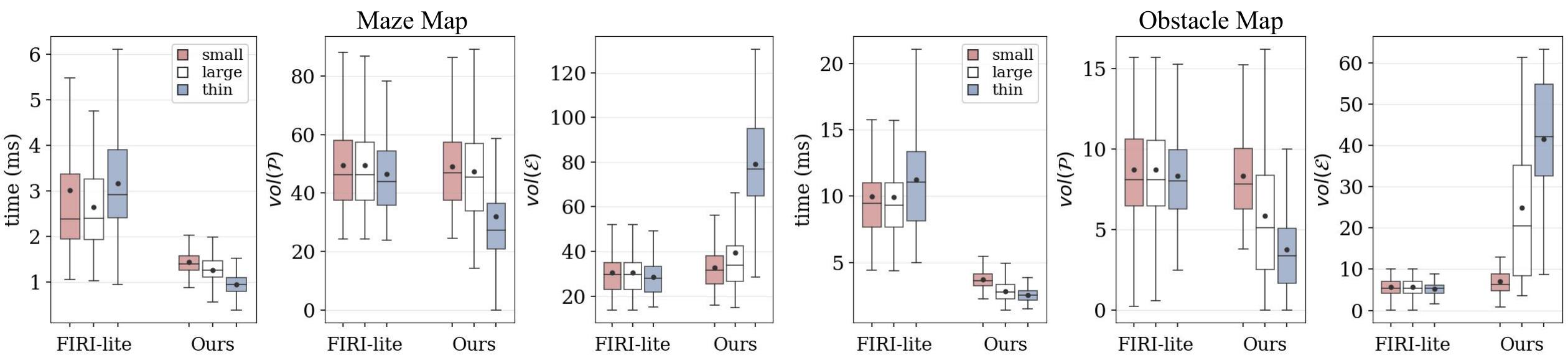}
    \vspace{-0.6cm}
      \caption{Comparison of performance by different ellipsoid initialization. }
      \label{fig:bench1}
      \vspace{-0.1cm}
\end{figure*}
As shown in Fig.~\ref{fig:bench1}, the thin and segment-aligned initialization generates the smallest final ellipsoid volumes.
It does not improve performance and consistently leads to poor local optima. 
As this elongated ellipsoid introduces overly restrictive geometric constraints, resulting in conservative polytope growth and limited expansion of the free space. 
Even in scenarios where line-segment containment is desired, segment-aligned initialization is suboptimal compared to simple spherical initializations.

\subsection{Benchmark Comparison}
We conduct 500 test cases for each scenario for a complete comparison.
The results are shown in Table~\ref{tab: comparison}.
The proposed method consistently achieves the lowest average computation time while maintaining comparable polytope volumes to FIRI-lite in all evaluated environments. 
Line-segment seeds are more challenging with higher computational cost and smaller polytopes across all methods, as they introduce additional geometric constraints to restrict feasible expansion. 
Despite this increased complexity, the proposed method remains the fastest in all scenarios and preserves full seed inclusion.
Galaxy is unstable in some real maps, and failed trials are excluded from the reported results.

\begin{table}[!ht]
\centering
\vspace{-0.2cm}
\hspace*{-0.1cm}
\caption{Comparison across different scenarios.}
\label{tab: comparison}
\vspace{-0.2cm}
\resizebox{0.99\textwidth}{!}{
\renewcommand\arraystretch{1.15}
\setlength{\tabcolsep}{2.6pt}
\begin{tabular}{p{1.3cm}p{1.5cm}>{\centering\arraybackslash}
p{2.5cm}p{1.4cm}p{2.4cm}p{2.4cm}p{2.4cm}>{\centering\arraybackslash}p{1.2cm}>{\centering\arraybackslash}p{1.7cm}}
\toprule
\textbf{Scenario} 
& \textbf{Seed Type}
& \textbf{Input (\#Local points)} 
& \textbf{Method} 
& \textbf{Time (ms)} 
& $\boldsymbol{\mathrm{vol}(\mathcal{P})\ (m^3)}$ 
& \textbf{\#Planes} 
& \textbf{Collision} 
& \textbf{Seed Inclusion} \\
\hline
\multirow{16}{*}{Maze} 
& \multirow{8}{*}{Point}  
& \multirow{3}{*}{Raw} 
& FIRI-lite    & 2.66 $\pm$ 1.39 & 47.78 $\pm$ 14.86 & 22.77 $\pm$ 4.89 & 0.000 & 1.000 \\
& & & RILS   & 2.05 $\pm$ 1.60 & 33.68 $\pm$ 16.96 & 20.32 $\pm$ 3.84 & 0.000 & 1.000 \\
& &  (8283 $\pm$ 1824) & Galaxy & 4.55 $\pm$ 1.23 & 28.48 $\pm$ 16.08 & 25.49 $\pm$ 7.06 & 0.000 & 1.000 \\
& & & Ours   & \textbf{1.32 $\pm$ 0.26} & 47.56 $\pm$ 14.87 & 21.35 $\pm$ 4.54 & 0.000 & 1.000 \\
\cline{3-9}
& 
& \multirow{3}{*}{Surface Grid} 
& FIRI-lite    & 2.77 $\pm$ 1.43 & 37.49 $\pm$ 14.53 & 19.43 $\pm$ 3.49 & 0.000 & 1.000 \\
& & & RILS   & 4.21 $\pm$ 2.20 & 28.63 $\pm$ 15.77 & 17.50 $\pm$ 3.35 & 0.000 & 1.000 \\
& & (14732 $\pm$ 2214) & Galaxy & 6.25 $\pm$ 1.83 & 18.75 $\pm$ 14.88 & 25.31 $\pm$ 8.82 & 0.002 & 1.000 \\
& & & Ours   & \textbf{2.07 $\pm$ 0.57} & 37.34 $\pm$ 14.60 & 18.99 $\pm$ 3.24 & 0.000 & 1.000 \\
\cline{2-9}
& \multirow{8}{*}{Line}   
& \multirow{3}{*}{Raw} 
& FIRI-lite    & 3.33 $\pm$ 1.44 & 48.79 $\pm$ 15.71 & 23.95 $\pm$ 4.73 & 0.000 & 1.000 \\
& & & RILS   & 2.92 $\pm$ 2.09 & 33.85 $\pm$ 16.67 & 20.41 $\pm$ 4.27 & 0.000 & 1.000 \\
& & (10646 $\pm$ 2290) & Galaxy & 4.93 $\pm$ 1.19 & 27.12 $\pm$ 15.19 & 24.96 $\pm$ 6.72 & 0.000 & 0.996 \\
& & & Ours   & \textbf{1.48 $\pm$ 0.28} & 48.25 $\pm$ 15.45 & 22.70 $\pm$ 4.45 & 0.000 & 1.000 \\
\cline{3-9}
& 
& \multirow{3}{*}{Surface Grid} 
& FIRI-lite    & 3.13 $\pm$ 1.14 & 38.26 $\pm$ 15.79 & 19.77 $\pm$ 3.41 & 0.000 & 1.000 \\
& & & RILS   & 3.78 $\pm$ 1.91 & 29.08 $\pm$ 16.52 & 17.69 $\pm$ 3.27 & 0.000 & 1.000 \\
& & (17527 $\pm$ 2478) & Galaxy & 5.46 $\pm$ 1.31 & 18.68 $\pm$ 15.15 & 24.96 $\pm$ 8.51 & 0.002 & 0.981 \\
& & & Ours   & \textbf{1.94 $\pm$ 0.35} & 38.23 $\pm$ 15.94 & 19.71 $\pm$ 3.49 & 0.000 & 1.000 \\
\hline
\multirow{16}{*}{Obstacle} 
& \multirow{8}{*}{Point}  
& \multirow{3}{*}{Raw} 
& FIRI-lite     & 9.70 $\pm$ 4.14 & 29.16 $\pm$ 13.02 & 29.78 $\pm$ 5.29 & 0.000 & 1.000 \\
& & & RILS   & 8.12 $\pm$ 2.90 & 21.31 $\pm$ 10.78 & 23.41 $\pm$ 3.97 & 0.000 & 1.000 \\
& & (48085 $\pm$ 13766) 
& Galaxy & 5.57 $\pm$ 1.26 & 20.18 $\pm$ 9.42 & 33.33 $\pm$ 6.42 & 0.000 & 0.999 \\
& & & Ours   & \textbf{4.19 $\pm$ 1.11} & 28.63 $\pm$ 13.04 & 27.97 $\pm$ 4.71 & 0.000 & 1.000 \\
\cline{3-9}
& 
& \multirow{3}{*}{Surface Grid} 
& FIRI-lite   & 6.99 $\pm$ 2.57 & 17.57 $\pm$ 10.87 & 29.71 $\pm$ 5.12 & 0.000 & 1.000 \\
& & & RILS   & 5.54 $\pm$ 1.52 & 11.66 $\pm$ 8.47 & 23.64 $\pm$ 3.87 & 0.000 & 1.000 \\
& & (31453 $\pm$ 5573) 
& Galaxy & 5.06 $\pm$ 0.84 & 11.04 $\pm$ 7.36 & 29.04 $\pm$ 6.41 & 0.000 & 1.000 \\
& & & Ours   & \textbf{3.29 $\pm$ 0.69} & 16.67 $\pm$ 10.82 & 28.06 $\pm$ 5.25 & 0.000 & 1.000 \\
\cline{2-9}
& \multirow{8}{*}{Line}   
& \multirow{3}{*}{Raw} 
& FIRI-lite    & 11.92 $\pm$ 4.93 & 28.42 $\pm$ 12.57 & 30.39 $\pm$ 5.02 & 0.000 & 1.000 \\
& & & RILS   & 10.28 $\pm$ 3.74 & 20.26 $\pm$ 9.93 & 23.74 $\pm$ 3.88 & 0.000 & 1.000 \\
& & (60113 $\pm$ 15479) 
& Galaxy & 6.30 $\pm$ 1.45 & 19.87 $\pm$ 9.50 & 32.99 $\pm$ 6.39 & 0.000 & 0.997 \\
& & & Ours   & \textbf{5.00 $\pm$ 1.41} & 27.78 $\pm$ 12.60 & 28.58 $\pm$ 4.84 & 0.000 & 1.000 \\
\cline{3-9}
& 
& \multirow{3}{*}{Surface Grid} 
& FIRI-lite   & 8.54 $\pm$ 3.06 & 22.02 $\pm$ 13.97 & 29.21 $\pm$ 5.00 & 0.000 & 1.000 \\
& & & RILS   & 5.74 $\pm$ 1.55 & 14.57 $\pm$ 10.48 & 23.53 $\pm$ 3.87 & 0.000 & 1.000 \\
& & (34777 $\pm$ 6276) 
& Galaxy & 5.29 $\pm$ 0.88 & 13.82 $\pm$ 9.59 & 29.85 $\pm$ 6.58 & 0.000 & 0.998 \\
& & & Ours   & \textbf{3.53 $\pm$ 0.56} & 21.14 $\pm$ 14.03 & 28.05 $\pm$ 4.76 & 0.000 & 1.000 \\
\hline
\multirow{8}{*}{Real} & \multirow{4}{*}{Point} & \multirow{3}{*}{Raw} & FIRI-lite & 2.26 $\pm$ 4.26 & 136.99 $\pm$ 47.30 & 13.64 $\pm$ 4.78 & 0.000 & 1.000 \\
& & & RILS & 1.36 $\pm$ 2.79 & 127.73 $\pm$ 54.27 & 13.06 $\pm$ 4.12 & 0.000 & 1.000 \\
& & (10618 $\pm$ 20559) & Galaxy & 1.32 $\pm$ 2.05 & 63.67 $\pm$ 39.84 & 28.83 $\pm$ 9.77 & 0.011 & 0.935 \\
& & & Ours & \textbf{1.19 $\pm$ 1.62} & 136.72 $\pm$ 47.31 & 13.67 $\pm$ 4.71 & 0.000 & 1.000 \\
\cline{2-9}
& \multirow{4}{*}{Line} & \multirow{3}{*}{Raw} & FIRI-lite & 3.45 $\pm$ 7.73 & 169.30 $\pm$ 62.97 & 14.35 $\pm$ 4.67 & 0.000 & 1.000 \\
& & & RILS & 2.20 $\pm$ 5.16 & 156.60 $\pm$ 72.17 & 13.56 $\pm$ 3.99 & 0.000 & 1.000 \\
& & (15085 $\pm$ 28657) & Galaxy & 1.81 $\pm$ 3.33 & 80.12 $\pm$ 54.09 & 28.58 $\pm$ 8.86 & 0.013 & 0.910 \\
& & & Ours & \textbf{1.67 $\pm$ 2.58} & 168.87 $\pm$ 63.24 & 14.27 $\pm$ 4.58 & 0.000 & 1.000 \\
\bottomrule
\end{tabular}
}
\end{table}

\begin{figure*}[!th]
      \centering
       \vspace{-0.0cm}
    \includegraphics[width=1\columnwidth]{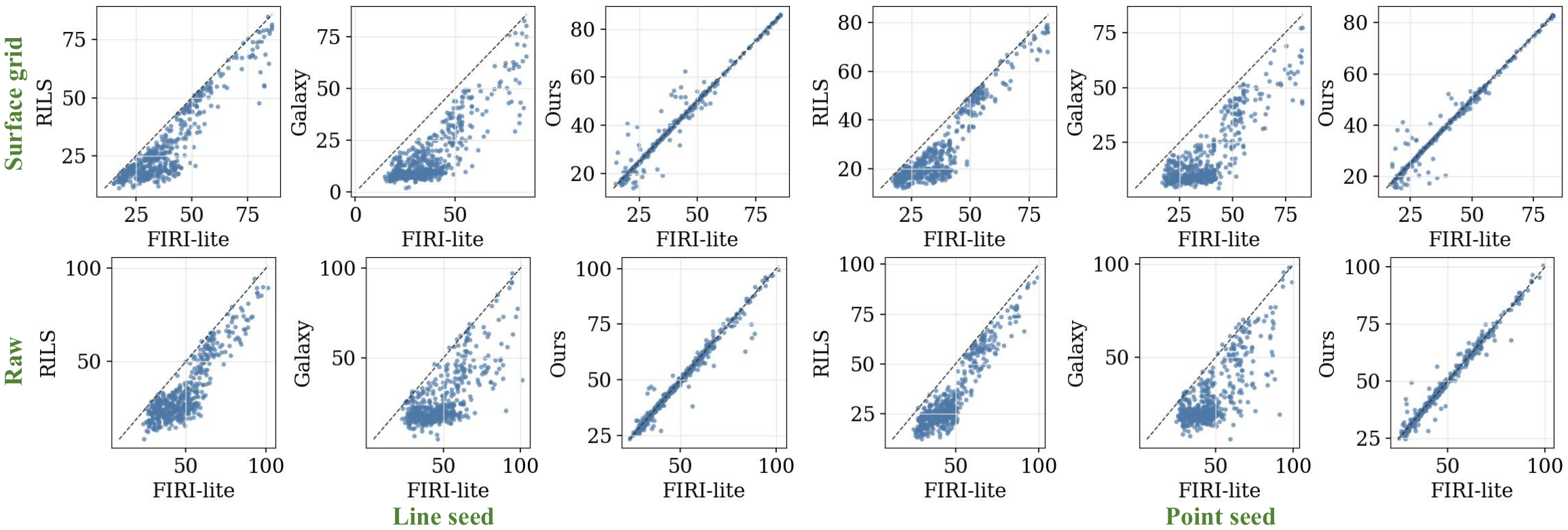}
    \vspace{-0.7cm}
      \caption{Pairwise comparison of polytope volumes in maze maps with FIRI-lite as the x-axis reference. }
      \label{fig: result}
      \vspace{-0.1cm}
\end{figure*}

To further illustrate the volume differences across each test, we plot the pairwise comparisons against FIRI-lite on maze maps in Fig.~\ref{fig: result}. 
Our method achieves performance similar to FIRI-lite, whereas Galaxy and RILS consistently underperform it. 
Although these two methods reduce computation time, this improvement is achieved at the expense of solution quality, while our method maintains performance without such degradation.

\subsection{Application to Agile Quadrotor Planning}

To further evaluate the ability for real-world deployment, we study the proposed framework under two challenging conditions, including single convex polytope generation from sparse, noisy incremental sensory updates and large-scale SFC construction for trajectory optimization.

For single polytope generation, we test on real event camera data collected on quadrotor platforms~\cite{zhu2018multivehicle}. 
Event cameras generate data with microsecond temporal resolution and low latency, which requires fast computation and robustness to sparse measurements.
We obtain depth estimates from stereo event camera data using the method of~\cite{niu2025esvo2}, with an average computation time of approximately $10$ ms.
Fig.~\ref{fig: pc2} shows convex polytopes constructed from event camera data over a short temporal window.
\begin{figure*}[!th]
      \centering
       \vspace{-0.1cm}
    \includegraphics[width=1\columnwidth]{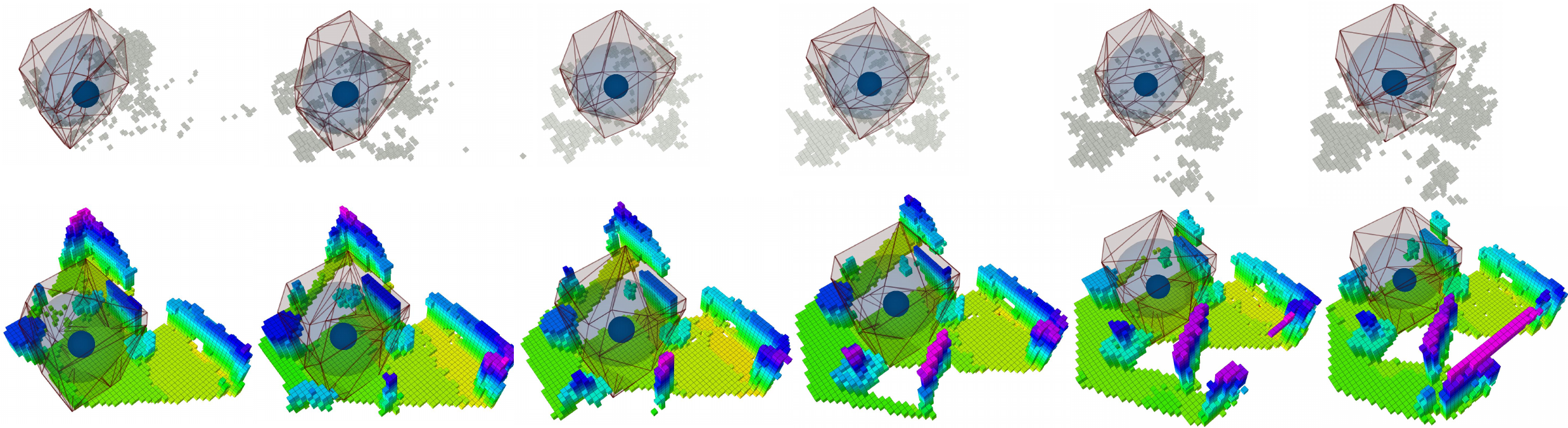}
    \vspace{-0.7cm}
      \caption{Convex polytopes over increasing temporal windows. The top row uses event-based data, and the bottom row uses ground-truth LiDAR.}
      \label{fig: pc2}
      \vspace{-0.1cm}
\end{figure*}

In a $10 \ \textrm{m} \times 10 \ \textrm{m}  \times 2 \ \textrm{m} $ workspace, we evaluate 200 trials with a 1 m local range and an average $229.26 \pm 413.97$ points. 
Our method achieves a runtime of $0.17 \pm 0.14$ ms and produces polytope volume $7.11 \pm 0.92$ $\text{m}^3$. 
We generate convex regions from event camera point clouds and evaluate collisions against the ground-truth point cloud. The average rate of volume in collision is $0.0226 \pm 0.0257$.
The results demonstrate that the proposed method achieves low-latency computation and enables fast incremental updates for agile planning.

For trajectory and SFC generation, the proposed method is integrated into a parallel optimization framework in~\cite{10935632}.
We employ global LiDAR data to generate SFCs with predefined waypoints.
The SFC by our method enables continuous piecewise polynomial trajectory optimization, as shown in Fig.~\ref{fig:mp}, and supports real-time deployment on onboard quadrotor hardware.

\begin{figure*}[!th]
      \centering
       \vspace{-0.0cm}
    \includegraphics[width=1\columnwidth]{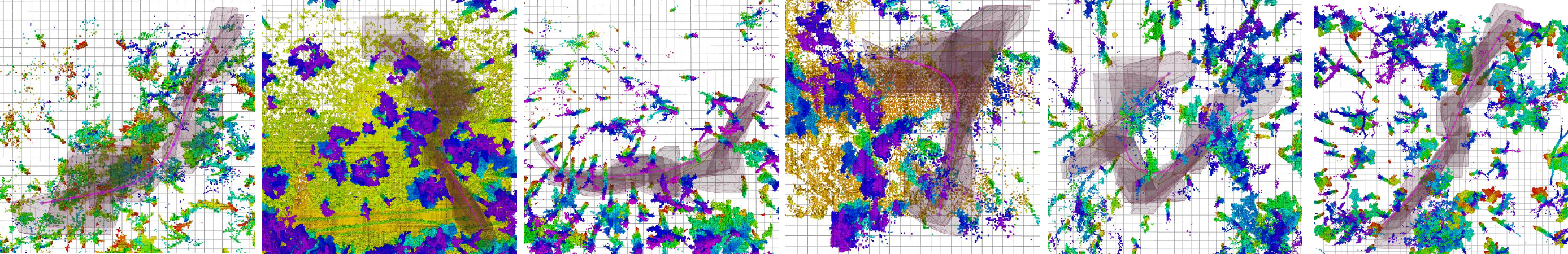}
    \vspace{-0.6cm}
      \caption{Piecewise polynomial trajectories (in magenta) optimized within SFC.}
      \label{fig:mp}
      \vspace{-0.4cm}
\end{figure*}

\section{Conclusion and Future Work} 
\label{sec:conclusion}

We propose a framework for superfast polytope inflation that leverages starshaped set as an efficient filtering to accelerate polytope extraction. 
By filtering obstacle points that are likely to become active supporting constraints, the proposed approach reduces the number of candidates considered during separating hyperplane generation.
The alternating polytope and ellipsoid procedure can be interpreted as an optimization over a reduced set of candidate normals, with the inscribed ellipsoid as a heuristic to guide normal selection rather than as a strict geometric approximation.
This design reduces redundant computation and improves robustness when operating on noisy sensor data and large-scale point clouds in real-world environments.
Future work will focus on GPU acceleration and learning-based filtering to further improve closed-loop onboard performance, especially for navigation with an event camera.

\vspace{-0.2cm}
\section*{Acknowledgments}

We greatly acknowledge the support of TILOS, funded by NSF Grant CCR-2112665, as well as the support from Pratik Chaudhari. 
We used language models (e.g., GPT) for text and code refinement.

\bibliography{references}
\end{document}